\documentclass{article}
\usepackage{kr}

\usepackage{times}
\usepackage{soul}
\usepackage{url}
\usepackage[hidelinks]{hyperref}
\usepackage[utf8]{inputenc}
\usepackage[small]{caption}
\usepackage{graphicx}
\usepackage{amsmath}
\usepackage{amsthm}
\usepackage{booktabs}
\usepackage{algorithm}
\usepackage{algorithmic}
\usepackage{todonotes}
\usepackage{amsfonts}
\usepackage{amssymb}
\usepackage{paralist}
\usepackage{stmaryrd}
\newtheorem{example}{Example}
\newtheorem{theorem}{Theorem}
\newtheorem{lemma}{Lemma}
\newtheorem{definition}{Definition}

\newcommand{\rom}[1]{\uppercase\expandafter{\romannumeral #1\relax}}

\usepackage{amsmath,nccmath}
\usepackage{xspace}
\usepackage{xcolor}

\newcommand{\myi}{(\emph{i})\xspace}
\newcommand{\myii}{(\emph{ii})\xspace}
\newcommand{\myiii}{(\emph{iii})\xspace}

\newcommand{\A}{\mathcal{A}}

\renewcommand{\L}{\mathcal{L}}

\renewcommand{\P}{\mathcal{P}}

\newcommand{\T}{\mathcal{T}}
\newcommand{\U}{\mathcal{U}}

\newcommand{\X}{\mathcal{X}}
\newcommand{\Y}{\mathcal{Y}}

\newcommand{\LTL}{{\sc ltl}\xspace}
\newcommand{\LTLf}{{\sc ltl}$_f$\xspace}

\newcommand{\true}{\mathit{true}}

\newcommand{\false}{\mathit{false}}

\newcommand{\Next}{\raisebox{-0.27ex}{\LARGE$\circ$}}
\newcommand{\Wnext}{\raisebox{-0.27ex}{\LARGE$\bullet$}}
\newcommand{\Until}{\mathop{\U}}

\newcommand{\trace}{{\mathsf{Trace}}}
\newcommand{\play}{{\mathsf{Play}}}
\newcommand{\wine}{\mathsf{Env}}
\newcommand{\wina}{\mathsf{Agn}}
\newcommand\run[1]{{\sf{Run}{#1}}}

\newcommand{\LTLftitle}{\textbf{\textsc{ltl}$_f$}\xspace}

\newcommand{\reach}{{\operatorname{Reach}}\xspace}
\newcommand{\safe}{{\operatorname{Safe}}\xspace}

\newcommand{\reachsafe}{{\operatorname{Reach--Safe}}\xspace}

\newcommand{\SetN}{\mathbb{N}}

\newcommand{\env}{{\operatorname{e}\xspace}}

\newcommand{\stopact}{\mathbf{stop}}

\newcommand{\stenv}{\gamma}

\newcommand{\stag}{{\sigma}}

\newcommand{\task}{\varphi_{\mathit{task}}}

\newcommand{\dutytask}{\varphi_{\operatorname{d}}}
\newcommand{\rightstask}{\varphi_{\operatorname{r}}}

\newcommand{\fdutytask}{\varphi_{\operatorname{fd}}}
\newcommand{\frightstask}{\varphi_{\operatorname{fr}}}

\newcommand{\fd}{{\operatorname{fd}}}
\newcommand{\fr}{{\operatorname{fr}}}

\newcommand{\ls}{\mathit{last}\xspace}

\newcommand{\DA}{{\sc{DA}}\xspace}
\newcommand{\afh}{{\operatorname{af(h)}\xspace}}
\newcommand{\hP}{{\hat{\P}}}

\newcommand{\duties}{{\operatorname{d}\xspace}}
\newcommand{\rights}{{\operatorname{r}\xspace}}

\title{Act for Your Duties but Maintain Your Rights}

\author{%
Shufang Zhu\footnote{Corresponding Author}
\and
Giuseppe De Giacomo\\
\affiliations
Sapienza University of Rome, Rome, Italy\\
\emails
\{zhu,degiacomo\}@diag.uniroma1.it
}

\begin{document}

\maketitle


\begin{abstract}
    Most of the synthesis literature has focused on studying how to synthesize a strategy to fulfill a task. 
    This task is a duty for the agent. In this paper, we argue that intelligent agents should also be equipped with rights, that is, tasks that the agent itself can choose to fulfill (e.g., the right of recharging the battery). The agent should be able to maintain these rights while acting for its duties.  
    We study this issue in the context of \LTLf synthesis: we give duties and rights in terms of \LTLf specifications, and synthesize a suitable strategy to achieve the duties that can be modified on-the-fly to achieve also the rights, if the agent chooses to do so. We show that handling rights does not make synthesis substantially more difficult, although it requires a more sophisticated solution concept than standard \LTLf synthesis.
    We also extend our results to the case in which further duties and rights are given to the agent while already executing.
\end{abstract}
\section{Introduction}
Consider the following example: we give to a robot the task of cleaning one by one series of rooms. The robot has a model of the world describing the effects of its actions, and, given the task specification, it synthesizes
a strategy to accomplish its cleaning task. However, in going after its task, the robot would like to be sure to be able to recharge its battery, if it thinks the battery level is getting too low. 
Both cleaning and recharging batteries are (temporally extended) tasks. Once the cleaning task is accepted, the agent must fulfill it, i.e., the cleaning task is a \emph{duty}. Instead, recharging the battery, is what we may call a \emph{right} of the robot, i.e., a task that the agent must be given the ability to fulfil, such that the agent itself can decide to actually fulfill or not. Handling both \emph{duties} and \emph{rights} is the issue studied in this paper.
    
The literature on strategy synthesis \cite{PnueliR89,finkbeiner16}, as well as the literature on planning~\cite{GhallabNauTraverso2016,2019Haslum}, focus only on fulfilling duties, without considering rights. 
Instead, our notion of rights is implicitly related to the notion of \emph{ability} studied in autonomous agents and reasoning about actions, see e.g.,~\cite{LesperanceLLS00}. Indeed, the ability of performing some task requires the existence of strategies for fulfilling the task, but not necessarily the decision to follow such strategy to actually fulfill it. In our case the agent has the \emph{ability} of satisfying also its \emph{rights} while executing the strategy for satisfying the duties, but actually satisfies the rights only if it wants to do so.
Also, talking about duties and rights calls for connections with \emph{obligations} and \emph{permissions} in Deontic Logic~\cite{GabbayForthcoming-GABHOD-2}. However, here we focus mainly on synthesis and leave the exact connection with Deontic Logic for future studies.

Specifically, in this paper, we study how to handle duties and rights in the context of Linear Temporal Logic on finite traces~(\LTLf), see \cite{DegVa13} for a survey.
\LTLf, on the one hand, allows for specifying a rich set of temporally extended specifications \cite{BacchusK00,SilvaML20}, and on the other
hand, focuses on finite traces, which makes it particularly suitable for specifying tasks of intelligent agents. Note that intelligent agents will not get stuck accomplishing a task for all their lifetime, but only for a finite (but unbounded) number of steps.

Technically, our starting point is 
\LTLf synthesis under environment specifications \cite{DegVa15,ADMRicaps19,DDPZ}. We assume the agent is acting in an environment that is specified through safety specifications, which can be thought of as an extension, possibly with non-Markovian features \cite{Gabaldon11}, of nondeterministic fully observable planning domains \cite{Cimatti03,GhallabNauTraverso2016}, as discussed, e.g., in  \cite{CamachoBM18,abs-1807-06777}. Wlog, we are going to use \LTLf, also for these environment safety specifications as in~\cite{Giacomo0TV021}. Over this environment, we give duties and rights to the agent, expressing both of them as arbitrary \LTLf specifications.
The problem that we want to solve is to 
synthesize a suitable strategy to achieve the duties that can be modified while in execution to achieve also the rights, if the agent chooses to do so.

We show that handling duties and rights is 2EXPTIME-complete, as standard \LTLf synthesis \cite{DegVa15}, though it requires a more sophisticated solution concept.
Essentially, we do not only compute the winning strategy as a transducer, but we guarantee that during its execution such a strategy never leaves the \emph{winning region}~(which technically captures the ability to fulfil) of both the duties and the rights. Moreover, by storing such winning region, we can readily build a further transducer representing  a strategy to fulfil also the rights at the moment the agent decides to do so while executing the first strategy.

We then study the case in which further duties and rights are given to the agent while the agent is already executing the strategy for the original duties and rights.
Handling further duties that are given while already executing a strategy is related to \emph{live synthesis}, which has been recently introduced in Formal Methods \cite{livesynthesis}. So as a by-product of our work, we devise a technique for live synthesis in \LTLf. 
We however extend this form of synthesis to handle also rights. We show that, even in this case, synthesis remains 2EXPTIME-complete, and we present techniques to effectively compute such kind of strategies with only a small overhead.

\section{Preliminaries}\label{sec:pre}
\subsection{\LTLftitle Basics}
\textit{Linear Temporal Logic on finite traces}~(\LTLf) is a specification language to express temporal properties on finite traces~\cite{DegVa13}. 
In particular, \LTLf has the same syntax as \LTL, one of the popular specification languages in Formal Methods, which is interpreted over infinite traces~\cite{Pnu77}. Given a set  of propositions $Prop$, the formulas of \LTLf are generated as follows: 
\centerline{$\varphi ::= a \mid (\neg \varphi) \mid (\varphi \wedge \varphi) \mid 
 (\Next \varphi) \mid (\varphi \Until \varphi),
$}

\noindent where $a \in Prop$ is an atom, $\Next$ for \emph{Next}, and $\Until$ for \emph{Until} are temporal operators. 
We make use of standard Boolean abbreviations, such as $\vee$~(or) and $\rightarrow$~(implies), $\true$ and $\false$.
Moreover, we have the following abbreviations for temporal operators, \emph{Eventually} as $\Diamond \varphi \equiv \true \Until \varphi$ and \emph{Always} as $\Box \varphi \equiv \lnot \Diamond \lnot \varphi$. In addition, we have the \emph{Weak Next} operator $\Wnext$ as abbreviation of $\Wnext \varphi \equiv \neg \Next \neg \varphi$.

A \textit{trace} $\pi = \pi_0\pi_1\ldots$ is a sequence of propositional interpretations~(sets), where for every $i \geq 0$, $\pi_i \in 2^{Prop}$ is the $i$-th interpretation of $\pi$. Intuitively, $\pi_i$ is interpreted as the set of propositions that are $true$ at instant $i$. A trace $\pi$ is an \textit{infinite} trace if $\ls(\pi) = \infty$, which is formally denoted as $\pi\in (2^{Prop})^{\omega}$; otherwise $\pi$ is a \textit{finite} trace, denoted as $\pi\in (2^{Prop})^{*}$. 
Moreover, by $\pi^k = \pi_0 \ldots \pi_k$ we denote the \emph{prefix} of $\pi$ up to the $k$-th instant. 
Sometimes we call a prefix of a trace \emph{history}. We denote by $\epsilon$ the empty prefix, i.e., the history of length $0$.
\LTLf formulas are interpreted over finite and nonempty traces. Given $\pi\in (2^{Prop})^{+}$, we define when an \LTLf formula $\varphi$ \emph{holds} at instant $i~(0 \leq i \leq \ls(\pi))$, written as $\pi, i \models \varphi$, inductively on the structure of $\varphi$, as:
\begin{compactitem}
	\item 
	$\pi, i \models a$ iff $a \in \pi_i\nonumber$ (for $a\in{Prop}$);
	\item 
	$\pi, i \models \lnot \varphi$ iff $\pi, i \not\models \varphi\nonumber$;
	\item 
	$\pi, i \models \varphi_1 \wedge \varphi_2$ iff $\pi, i \models \varphi_1 \text{ and } \pi, i \models \varphi_2\nonumber$;
	\item 
	$\pi, i \models \Next\varphi$ iff $i< \ls(\pi)$ and $\pi,i+1 \models \varphi$;
	\item 
	$\pi, i \models \varphi_1 \Until \varphi_2$ iff $\exists j$ such that $i \leq j \leq \ls(\pi)$ and $\pi,j \models\varphi_2$, and $\forall k, i\le k < j$, we have that $\pi, k \models \varphi_1$.
\end{compactitem}
We say $\pi$ \emph{satisfies} $\varphi$, written as $\pi \models \varphi$, if $\pi, 0 \models \varphi$.

\subsection{\LTLftitle for Safety Properties}
\emph{Safety properties} assert that \emph{undesired things never happen}, i.e., a trace always behaves within some allowed boundaries. Thereby, safety properties exclude traces that can be violated by a ``bad" finite prefix. Typically, safety properties are captured as \LTL formulas~\cite{KupfermanVa01}, interpreted over infinite traces. Alternatively, it has been shown in~\cite{Giacomo0TV021} that, one can use \LTLf formulas to capture safety properties over both of finite and infinite traces, by applying an alternative notion of satisfaction that interprets an \LTLf formula over \emph{all} prefixes of a trace. 

\begin{definition}
A (finite or infinite) trace $\pi$ satisfies an \LTLf formula $\varphi$ on \emph{all prefixes}, denoted $\pi \models_\forall \varphi$, if every nonempty finite prefix of $\pi$ satisfies $\varphi$. That is, $\pi^k = \pi_0 \pi_1 \ldots \pi_k \models \varphi$, for every $0 \leq k \leq \ls(\pi)$.
\end{definition}

Moreover, all safety properties expressible in \LTL, i.e., all first-order (logic) safety properties~\cite{LichtensteinPZ85}, can be specified using  \LTLf on all prefixes.
\begin{theorem}{\cite{Giacomo0TV021}}
	Every first-order safety property can be expressed as an \LTLf formula on all prefixes.	
\end{theorem}

\subsection{\LTLf Synthesis with Safety Env Specs} 
Reactive synthesis can be viewed as a game between the \textit{environment} and the \textit{agent}, contrasting each other by controlling two disjoint sets of variables $\X$ and $\Y$, respectively. The goal of reactive synthesis is to synthesize an agent strategy such that no matter how the environment behaves, the combined trace from two players satisfy desired properties~\cite{PnueliR89}. In standard synthesis, the agent assumes the environment to be free to choose an arbitrary move at each step, but in AI typically the agent has some knowledge of how the environment works. The environment knowledge that the agent knows apriori is called \emph{environment specification} \cite{DegVa15,ADMRicaps19,DDPZ}.
 
In particular, we focus on the environment specifications that  are formed by safety properties. In this way our environment specifications can be thought as an extension of fully observable nondeterministic domains~\cite{Cimatti03,GhallabNauTraverso2016}, see also \cite{Rintanen04a}. Formally, an environment specification is an \LTLf safety formula $env$, while the agent task is expressed as a standard \LTLf formula $\task$. We describe the synthesis problem as a tuple $\P = (env, \task)$. Note that for simplicity, we do not explicitly list $\X$ and $\Y$ here, since they are given as inputs by default and thus are clear from the context.

An environment strategy is a function $\stenv: (2^{\Y})^* \to 2^\X$, and an agent strategy is a function $\stag: (2^{\X})^+ \to 2^\Y$. A trace $\pi = (X_0 \cup Y_0) (X_1 \cup Y_1) \dots \in (2^{\X \cup \Y})^\omega$, is \emph{compatible} with an environment strategy $\stenv$ if $\stenv(\epsilon) = X_0$ and $\stenv(Y_0Y_1\ldots Y_i) = X_{i+1}$ for every $i$. A trace $\pi$ being compatible with an agent strategy $\stag$ is defined analogously. 
Sometimes, we write $\stag(\pi^k)$ instead of $\stag(X_0 X_1 \cdots X_k)$ for simplicity. We denote the unique infinite sequence that is compatible with $\stenv$ and $\stag$ as $\trace(\stenv, \stag)
$. We also generalize these definitions to finite traces in the obvious way.

Turning to agent strategies, wlog, we require them to be \emph{stopping}, i.e., we require the agent to perform a mandatory stop action. More specifically, every action of the agent is considered as an assignment over $\Y$, and $\stopact$ is one of them. For convenience, wlog, $\stopact$ is encoded as an assignment where all variables in $\Y$ are set to $false$, i.e., $\stopact= \bigwedge_{y \in \Y} \neg y$. 

\begin{definition}{\cite{DDPZ}}
A \emph{stopping} agent strategy is a function $\stag: (2^\X)^+ \rightarrow 2^\Y$, such that for every trace $\pi \in (2^{\X \cup \Y})^{\omega}$ that is compatible with $\stag$, there exists $i \in \SetN$ such that $\stag(\pi^j) = \stopact$ for every $j \geq i$ and $\stag(\pi^{h}) \neq \stopact$ for every $h < i$.
\end{definition}

Having \emph{stopping} agent strategies, we define the \emph{play} induced by given $\stenv$ and $\stag$ as the finite prefix of $\trace(\stenv, \stag)$ that ends right before the first $\stopact$, denoted by $\play(\stenv, \stag)$. Formally, $\play(\stenv, \stag) = (X_0 \cup Y_0) (X_1 \cup Y_1) \dots (X_i \cup Y_i)$ where $Y_{i+1} = \stopact$, and $Y_j \neq \stopact$ for every $0 \leq j \leq i$.

Given an environment safety specification $env$, which is an \LTLf formula, an environment strategy $\stenv$ \emph{enforces} $env$, written $\stenv \ \rhd \ env$, if for every agent strategy $\stag$, it holds that $\trace(\stenv, \stag) \models_\forall env$. We denote the set of environment strategies enforcing $env$ by $\llbracket env \rrbracket$.

\begin{definition}{\cite{DDPZ}}
The problem of synthesis is described as a tuple $\P = (env, \task)$. Realizability of $\P$ checks whether there exists an agent strategy $\stag$ such that $\forall \stenv \in \llbracket env \rrbracket$, $\play(\stenv, \stag) \models \task$. Synthesis of $\P$ computes such a strategy if exists.
\end{definition} 
As usual, we require that $env$ must be \emph{environment realizable}, i.e., $\llbracket env \rrbracket$ is nonempty. 
As shown in~\cite{DDPZ,Giacomo0TV021}, this kind of synthesis can be solved through a reduction to a suitable two-player game constructed from \LTLf formulas $env$ and $\task$, which takes 2EXPTIME. The problem itself is 2EXPTIME-complete.

\subsection{Two-player Games}
A \emph{two-player game} is a game between the \emph{environment} and the \emph{agent}, controlling two disjoint sets of variables $\X$ and $\Y$, respectively. The game is described by a \emph{deterministic automaton} (\DA), which is a tuple $\A = (2^{\X \cup \Y}, Q, I, \delta, \alpha)$, where $2^{\X \cup \Y}$ is the alphabet, $Q$ is a finite set of states, $I \in Q$ is the initial state, $\delta : Q \times 2^{\X \cup \Y} \rightarrow Q$ is the transition function, and $\alpha \subseteq Q^\omega$ is an acceptance condition.
Given an infinite word $\pi = \pi_0 \pi_1 \pi_2 \ldots \in (2^{\X \cup \Y})^\omega$, the \emph{run} $\rho = \run(\A, \pi)$ of $\A$ on $\pi$ is an infinite sequence $\rho = q_0 q_1 q_2 \ldots \in Q^\omega$, where $q_0 = I$ and $q_{i+1} \in \delta(q_i, \pi_i)$ for every $i \geq 0$. The run of $\A$ on a finite prefix $\pi^k$ is defined analogously, and so $\run(\A, \pi^k) = q_0 q_1 q_2 \ldots q_{k+1}$.
A run $\rho$ is \emph{accepting} if $\rho \in \alpha$.
The \emph{language} of $\A$, denoted by $\L(\A)$, is the set of words accepted by $\A$.
In this work, we specifically consider the following acceptance conditions:
\begin{compactitem}
    \item
	\emph{Reachability.} Given a set $R \subseteq Q$,  
	$\reach(R)= \{ q_0 q_1 q_2 \ldots \in Q^\omega  \mid \exists k \geq 0: q_k \in R\}$, i.e., a state in $R$ is visited at least once.
	\item
	\emph{Safety.}  Given a set $S \subseteq Q$, $\safe(S)= \{ q_0 q_1 q_2 \ldots \in Q^\omega \mid \forall k \geq 0: q_k \in S\}$, i.e., only states in $S$ are visited.
	\item
	\emph{Reachability-Safety.}
	Given two sets $R, S \subseteq Q$,
	$\reachsafe(R, S) =  \{ q_0 q_1 q_2 \ldots \in Q^\omega \mid \exists k \geq 0 : q_k \in R \mbox{ and } \forall j, 0 \leq j \leq k : q_j \in S \}$, i.e.,
	a state in $R$ is visited at least once, and until then only states in $S$ are visited.
\end{compactitem}
Notably, a \DA with reachability acceptance condition defines a deterministic finite automaton (DFA).
Depending on the actual acceptance condition $\alpha$, we get reachability, safety, or reachability-safety games.
\begin{theorem}{~\cite{Giacomo0TV021}}\label{thm:reachsafe-as-reach}
    Reachability-safety game can be solved by a linear-time reduction to a reachability game.
\end{theorem}
Given a game $\A=(2^{\X \cup \Y}, Q, I, \delta, \alpha)$ defined above, an agent strategy $\stag$ is \emph{winning} if $\forall \stenv. \trace(\stenv, \stag) \in \L(\A)$. A state $q \in Q$ is an agent~(resp.~environment) \emph{winning} state if the agent~(resp.~environment) has a winning strategy in $\A'=(2^{\X \cup \Y}, Q, q, \delta, \alpha)$, i.e., same structure but a new initial state $q$. By $\wina$~(resp.~$\wine$) we denote the set of all agent~(resp.~environment) winning states, also called the agent (environment) \emph{winning region}. 
All the games defined above are \emph{determined}, i.e., $q \in Q$ is an agent winning state~($q \in \wina$) iff $q$ is not an environment winning state~($q \notin \wine$)~\cite{Mar75}.  

\section{Synthesis with Duties and Rights}\label{sec:synthesis-rights}
In a common synthesis setting, the agent typically follows a strategy blindly. In other words, any action that the agent performs is expected to serve the task.
In this paper, we would like to assign more freedom to the agent, and thus look into the scenario where the agent has its own rights of doing some work in its own favor.
For example, along the way in cleaning a series of rooms, the robot should remain able to recharge the battery, if it thinks the battery level is getting too low. 
Note that the robot must make sure that the rooms are cleaned when it stops, no matter whether it chooses or not to recharge the battery while cleaning. 

In this synthesis setting, we divide agent tasks into two types: \emph{duties}, expressed as an \LTLf formula $\dutytask$, specifying the mandatory tasks that the agent has to accomplish; \emph{rights}, expressed as an \LTLf formula $\rightstask$, specifying the optional tasks that the agent has the right to decide whether to accomplish. To make sure that the agent can purse $\rightstask$ whenever it chooses to do so, the agent should be equipped with the ability of accomplishing also $\rightstask$ while achieving $\dutytask$.

We start with defining a strategy that enforces a specification $\varphi$ with respect to a history $h$, indicating the moment that the agent chooses to pursue $\varphi$.
\begin{definition}\label{def:strategy-with-h}
    Let $\varphi$ be an \LTLf formula and $h \in (2^{\X \cup \Y})^*$ be a history. An agent strategy $\stag$ \emph{enforces} $\varphi$, with respect to history $h$, denoted by $\stag \ \rhd_h \ \varphi$, if $\forall \stenv \in \llbracket env \rrbracket$ such that $\play(\stenv, \stag)$ has $h$ as a prefix, we have that $\play(\stenv, \stag) \models \varphi$.
\end{definition}
It should be noted that, in Definition~\ref{def:strategy-with-h}, we only consider the cases where both the environment strategy $\stenv$ and the agent strategy $\stag$ are compatible with $h$. That is, $h = (X_0 \cup Y_0) (X_1 \cup Y_1) \dots (X_i \cup Y_i) \in (2^{\X \cup \Y})^*$ is such that for every $0 \leq j \leq i$: $\stag(X_0 X_1 \cdots X_j) = Y_{j}$ and $Y_j \neq \stopact$; $\stenv(\epsilon) = X_0$ and $\stenv(Y_0 Y_1 \cdots Y_j) = X_{j + 1}$. 

Computing such a strategy is analogous to computing a strategy that enforces $\varphi$. We start with computing $\llbracket env \rrbracket$ by taking the following steps:
\begin{compactenum}
    \item
    Build $\A_\env = (2^{\X \cup \Y}, Q_\env, I_\env, \delta_\env, \safe(S))$ that accepts a trace $\pi$ iff $\pi \models_\forall env$. 
    \item
    Solve the safety game on $\A_\env$ for the environment, thus obtaining the \emph{environment winning region} $\wine$.
    \item
    Restrict $\A_\env$ with $\wine$ into $\A'_\env=(2^{\X \cup \Y}, \wine, I_\env, \delta'_\env, $ $ \safe(\wine))$, $\delta'_\env(q, X \cup Y) = \text{undefined}$, if $\exists Y' \in 2^{\Y}. \delta_\env(q,X \cup Y')  \not\in \wine$; $\delta'_\env(q, X \cup Y) = \delta_\env(q,X \cup Y)$ otherwise.
\end{compactenum}
It should be noted that for safety games, there exists a unique ``nondeterministic" strategy that can capture the set of all winning strategies. This strategy can be intuitively interpreted as a ``staying in the winning region" strategy~\cite{BernetJW02}. Therefore, $\A'_\env$ precisely captures $\llbracket env \rrbracket$.

Now we translate \LTLf formula $\varphi$ into DA $\A_\varphi = (2^{\X \cup \Y}, Q_\varphi, I_\varphi, \delta_\varphi, \reach(R_\varphi))$ that accepts a trace $\pi$ iff $\pi^k \models \varphi$ for some $k \geq 0$, and take the product of $\A'_\env$ and $\A_\varphi$ into $\A = (2^{\X \cup \Y}, Q, I, \delta, \reach(R))$, where $Q = \wine \times Q_\varphi$, $I = (I_\env, I_\varphi)$, $\delta((q_1,q_2), X\cup Y) = (\delta_\env(q_1,X \cup Y), \delta_\varphi(q_2,X \cup Y))$, and $R = R_\varphi$~(for simplicity, we omit the projection of states in $R$ to $R_\varphi$ here, we do the same later for similar usage). Indeed, $\delta((q_1,q_2), X\cup Y)=\text{undefined}$, if $\delta_\env(q_1,X \cup Y) = \text{undefined}$. At the end, solve a reachability game on $\A$ for the agent via a least fixpoint computation and obtain the \emph{agent winning region} 
$\wina_\varphi = \bigcup_{0 \leq l \leq u} \wina_\varphi^l$, where $\wina_\varphi^l$ are the ``approximates" of the fixpoint computation.
Clearly, if $\A$ does not have an agent winning strategy, i.e., $I \not\in \wina_\varphi$, or $\run(\A, h)$ does not always visit states in $\wina_\varphi$, then there does not exist an agent strategy enforcing $\varphi$ with respect to $h$. Otherwise, we abstract $\stag$ enforcing $\varphi$ with respect to $h$, by first restricting $\stag$ to be compatible with $h$, considering only the environment strategies that are compatible with $h$, then following the least fixpoint computation to get closer to $R_\varphi$ at every step until reaching $R_\varphi$. At the end, $\stag$ keeps playing $\stopact$ right after reaching $R_\varphi$.
The correctness of the construction is justified by the following lemma, which is easy to prove by construction.
\begin{lemma}\label{lem:strategy-enforce-with-h}
    Let $\varphi$ be an \LTLf formula, $h \in (2^{\X \cup \Y})^*$ be a finite history, and $\stag$ be constructed as above. Then $\stag$ enforces $\varphi$ with respect to $h$.
\end{lemma}
\begin{proof}
By construction, we compute the set of agent winning states $\wina_\varphi = \bigcup_{0 \leq l \leq u} \wina_\varphi^l$ through the reachability game on $\A_\varphi$. Moreover, if $\run(\A_\varphi,h)$ does not always visit states in $\wina_\varphi$, there does not exist an agent strategy enforcing $\varphi$ with respect to $h$. Otherwise, $\run(\A_\varphi,h)$ leads to a state $q \in \wina_\varphi$, from where there exists an agent strategy $\hat{\stag}$ that guides the play to $R$. At the end, $\hat{\stag}$ keeps playing $\stopact$ right after visiting $R$.

$\stag$ is computed by first following $h$ to $q$, and then following $\hat{\stag}$ to get closer to $R$ at every step until reaching $R$. Formally, for every $\xi^k \in (2^\X)^+$ that is compatible with $h = (X_0 \cup Y_0) (X_1 \cup Y_1) \dots (X_i \cup Y_i) \in (2^{\X \cup \Y})^*$

$\stag(\xi^k) = 
\begin{cases}
    Y_k & \text{ if } 0 \leq k \leq i, \\
    \hat{\stag}(\iota) & \text{ if } \xi^k = h \cdot \iota.
\end{cases}
$

Indeed, when the play reaches $R$, $\varphi$ is satisfied. Hence, it holds that $\forall \stenv \in \llbracket env \rrbracket$ such that $\play(\stenv, \stag)$ has $h$ as a prefix, we have that $\play(\stenv, \stag) \models \varphi$.
\end{proof}

The following theorem shows that computing a strategy that enforces $\varphi$ with respect to a history $h$ is not more difficult than computing a strategy that just enforces $\varphi$.
\begin{theorem}\label{thm:enforce-with-h}
    Let $\varphi$ be an \LTLf formula and $h \in (2^{\X \cup \Y})^*$ be a history. Computing an agent strategy that enforces $\varphi$ with respect to history $h$ is 2EXPTIME-complete in $\varphi$.
\end{theorem}
\begin{proof} 
    We prove from the following two aspects.
    \item{Membership.}
    Constructing the automata from \LTLf formulas $\varphi$ and $env$ contributes to the main computational complexity, which takes 2EXPTIME. 
    The final reachability game can be constructed in polynomial time in the size of the DAs $\A_{\env}$, $\A_\varphi$. Solving the final reachability game can be done in linear time in the size of the game arena.
    \item{Hardness.}
    Note that as a special case of this problem, we have standard \LTLf synthesis, by considering $h$ as an empty trace. And \LTLf synthesis itself is 2EXPTIME-complete~\cite{DegVa15}.
\end{proof}

For the synthesis setting that allows agent rights $\rightstask$ while pursuing $\dutytask$, we expect an agent strategy being able to enforce $\dutytask$, and along the execution until then, the agent is always able to enforce also $\rightstask$, i.e., to enforce $\dutytask \wedge \rightstask$.
\begin{definition}
Agent strategy $\stag$ enforcing $\dutytask$ is right-aware for $\rightstask$ if $\forall \stenv \in \llbracket env \rrbracket$:
\begin{itemize}
    \item 
    $\play(\stenv, \stag) \models \dutytask$;
    \item
    for every prefix $h$ of $\play(\stenv, \stag)$, there exists an agent strategy $\stag_{h}$ that enforces $\dutytask \wedge \rightstask$ with respect to history $h$, i.e., $\stag_{h} \ \rhd_h \ \dutytask \wedge \rightstask$.
\end{itemize}
\end{definition}

The problem of \LTLf synthesis with duties and rights is defined as follows.
\begin{definition}[\LTLf synthesis with duties and rights]\label{def:synthesis}
The problem is described as a tuple $\P= (env, \dutytask, \rightstask)$, where $env$ is an \LTLf formula specifying the environment safety specification, $\dutytask$ and $\rightstask$ are \LTLf formulas specifying the duties and rights, respectively. Realizability of $\P$ checks whether there exists an agent strategy $\stag$ enforcing $\dutytask$ that is right-aware for $\rightstask$. Synthesis of $\P$ computes a strategy $\stag$ if exists.
\end{definition}  

This class of synthesis problem is able to naturally reflect the problem structure of many autonomous agent applications. We illustrate this with a relatively simple example.
\begin{example}\label{example:1}
Consider a cleaning robot working in a circular hallway, where the charging station is located close to the entrance. Suppose the robot gets assigned a duty of ``cleaning room A" $\dutytask = \Diamond (\neg \textit{Dust\_A} \wedge \textit{RobotOut\_A})$, together with the rights of ``fully charging battery" $\rightstask = \Diamond (\textit{BatteryFull})$. In this hallway, the robot has two strategies to enforce $\dutytask$: 
\begin{compactenum}
    \item
    Take the direction that passes the charging station to room A and clean it. The remaining battery after enforcing $\dutytask$ still allows the robot to reach the charging station;
    \item
    Take the other direction to reach room A and clean it. The remaining battery after enforcing $\dutytask$ is not enough for the robot to reach the charging station.
\end{compactenum}
Although both strategies allow the robot to enforce $\dutytask$, only strategy (1) allows the robot to enforce $\dutytask$ and be right-aware of $\rightstask$ ``fully charging battery". 
\end{example}

\subsection{Synthesis Technique}\label{sec:3-1}
Following the construction explained above, we can compute $\llbracket env \rrbracket$ and represent it as $\A'_\env = (2^{\X \cup \Y}, \wine, I_\env, \delta'_\env, \safe(\wine))$. Moreover, we know that both duties $\dutytask$ and rights $\rightstask$ can be represented by DAs $\A_{\duties}$ and $\A_{\rights}$ with reachability conditions, respectively. The crucial difference is that, apart from achieving $\dutytask$ through a reachability game on $\A_{\duties}$, agent rights allow the agent to decide whether to achieve $\rightstask$. To do so, the agent should have the ability to make such decision, which can be naturally captured by the agent winning region of the reachability game on $\A_{\rights}$.

Given the synthesis problem $\P= (env, \dutytask, \rightstask)$, we have the following: regardless of which strategy the environment chooses to enforce $env$, thus  staying in $\A'_\env$, the desired agent strategy must make sure that the generated trace satisfies the reachability condition of $\A_{\duties}$, and that  if the agent decides to pursue also $\rightstask$, there exists a strategy that the agent can take to satisfy the reachability conditions of $\A_{\duties}$ and $\A_{\rights}$.

To synthesize such a strategy, we do the following: 
\myi compute the agent winning region $\wina_{\rights}$, from where the agent is able to lead the trace 
to satisfy the reachability conditions of $\A_{\duties}$, also $\A_{\rights}$; 
\myii compute an agent winning strategy $\stag$ s.t.
for every $\stenv \in \llbracket env \rrbracket$, $\trace(\stenv, \stag)$ satisfies the reachability condition of $\A_{\duties}$ by visiting \emph{only} states in $\wina_{\rights}$. In this way, the agent maintains the ability of also satisfying the reachability conditions of $\A_{\rights}$.
We now elaborate on every step.

\smallskip
\smallskip
\noindent\textbf{Step 1.} 
Compute $\wina_{\rights}$. 
Build $\A_{\duties} = (2^{\X \cup \Y}, Q_{\duties}, I_{\duties},$ $ \delta_{\duties}, \reach(R_{\duties}))$ that accepts a trace $\pi$ iff $\pi^k \models \varphi_{\duties}$ for some $k \geq 0$, and $\A_{\rights} = (2^{\X \cup \Y}, Q_{\rights}, I_{\rights}, \delta_{\rights}, \reach(R_{\rights}))$ that accepts a trace $\pi$ iff $\pi^k \models \varphi_{\rights}$ for some $k \geq 0$.
Take the product of $\A'_\env$, $\A_{\duties}$, and $\A_{\rights}$ into $\A = (2^{\X \cup \Y}, Q, I,$ $ \delta, \reach(R))$, where $Q = \wine \times Q_{\duties} \times Q_{\rights}$, $I = (I_\env, I_{\rights}, I_{\duties})$, $\delta((q_1,q_2,q_3), X\cup Y) = (\delta_\env(q_1,X \cup Y), \delta_{\duties}(q_2,X \cup Y), \delta_{\rights}(q_3,X \cup Y))$, and $R = R_{\duties} \cap R_{\rights}$. Indeed, $\delta((q_1,q_2,q_3), X\cup Y) = \text{undefined}$, if $\delta_\env(q_1,X \cup Y) = \text{undefined}$. 
At the end, solve a reachability game on $\A$ for the agent via a least fixpoint computation, thus obtaining $\wina_{\rights} = \bigcup_{0 \leq i \leq m} \wina_{\rights}^i$. If $I \not\in \wina_{\rights}$, return ``unrealizable".

\begin{lemma}\label{lem:realizability}
Let $\P$ be a problem of \LTLf synthesis with duties and rights, and $\wina_{\rights}$ the agent winning region of the reachability game on $\A = (2^{\X \cup \Y}, Q, I, \delta, \reach(R_{\duties} \cap R_{\rights}))$ computed as above. Then $\P$ is realizable iff $I \in \wina_{\rights}$.
\end{lemma}
\begin{proof}
We prove the lemma in both directions.
\item{($\Leftarrow$)} We need to show that if $I \in \wina_{\rights}$, then $\P$ is realizable.
By construction, $I \in \wina_{\rights}$ shows that there exists an agent strategy $\stag$ such that, for every $\stenv \in \llbracket env \rrbracket$, $\pi = \trace(\stenv, \stag)$ is such that $\pi \in \L(\A)$. 
That is to say, $\pi^k \models \dutytask \wedge \rightstask$ for some $k \geq 0$, thus it also holds that $\pi^k \models \dutytask$. In this case, $\stag$ starts playing $\stopact$ after $\pi^k$, and so we have $\play(\stenv, \stag) = \pi^k$. Moreover, for every prefix $h$ of $\play(\stenv, \stag)$, we can construct an agent strategy $\stag_h$ that works exactly the same as $\stag$, which indeed enforces $\dutytask \wedge \rightstask$ with respect to history $h$. 

\item{($\Rightarrow$)} We prove by contradiction. If $I \not\in \wina_{\rights}$, then there does not exist an agent winning strategy of the reachability game on $\A$. Suppose the agent decides to pursue also $\rightstask$ at the very beginning, then the agent does not have a strategy that enforces $\dutytask \wedge \rightstask$ with respect to history $h = \epsilon$, i.e., empty trace. Hence, $\P$ is unrealizable.
\end{proof}

\noindent\textbf{Step 2.}  Compute strategy $\stag$. Note that $\stag$ needs to lead the play to reach $R_{\duties}$ by visiting states in $\wina_{\rights}$ only. First, we define a new \DA with reachability-safety condition $\A_1 = (2^{\X \cup \Y}, Q, I, \delta, \reachsafe(R_{\duties}, \wina_{\rights}))$ from $\A = (2^{\X \cup \Y}, Q, I, \delta, \reach(R_{\duties} \cap R_{\rights}))$. It has been shown in~\cite{Giacomo0TV021} that $\A_1$ can be reduced to a new DA $\A_1' = (2^{\X \cup \Y}, Q, I, \delta', \reach(R'))$ with $\delta'$ and $R'$ as follows:
\begin{compactitem}
\item $\delta'(q, X \cup Y) = \begin{cases}
    \delta(q, X \cup Y) &\text{ if $q \in \wina_{\rights}$} \\
    q &\text{ if $q \not \in \wina_{\rights}$}
\end{cases}$
\item $R' = R_{\duties} \cap \wina_{\rights}$
\end{compactitem}
\noindent
Intuitively, the only change in $\delta'$ is to turn all non-safe states~(states not in $\wina_{\rights}$) into sink states, while $R'$ requires reaching a goal state~(a state in $R_{\duties}$) that is also safe~(i.e., it is in $\wina_{\rights}$).
Then we solve a reachability game on $\A_1'$ via a least fixpoint computation and obtain $\wina = \bigcup_{0 \leq j \leq n} \wina^j$. Note that $I \in \wina$ indeed holds, which is guaranteed by the reachability game for computing $\wina_{\rights}$ in the previous step. Finally, we define a strategy generator based on $\wina = \bigcup_{0 \leq j \leq n} \wina^j$, represented as a transducer 
$\T = (2^{\X \cup \Y}, Q, I, \varrho, \tau)$, where
\begin{itemize}
    \item
    $2^{\X \cup \Y}$, $Q$ and $I$ are the same as in $\A$;
    \item
    $\varrho: Q \times 2^{\X} \rightarrow 2^Q$ is the transition function such that 
    $\varrho (q, X) = \{q'\mid q' = \delta (q, X \cup Y) \text{ and } Y \in \tau(q)\};$
	\item 
	$\tau: Q \times 2^\X \rightarrow 2^{2^{\Y}}$ is the output function s.t. $\forall X \in 2^\X$, 
	$\tau(q, X) = \{ Y \mid \delta(q, X \cup Y) \in \wina^j \}$ if $q \in (\wina^{j+1} \backslash \wina^j)$, otherwise $\tau(q, X) = 2^\Y$.
\end{itemize}

This transducer generates an agent strategy $\stag: (2^\X)^+ \to 2^{\Y}$ in the following way: for every $\xi^k \in (2^{\X})^+~(k \geq 0)$

$
\stag(\xi^k) = 
\begin{cases}
    \stopact & \text{ if } \run(\A, \pi^{k-1}) \text{ visited } R_{\duties}, \\
    Y \in \tau(q_k, X_k) & \text{ otherwise. }
\end{cases}
$
where $\run(\A, \pi^{k-1}) = q_0q_1q_2 \ldots q_k$ s.t. $q_0 = I$, and $\pi^{k-1} = (X_0 \cup Y_0) (X_1 \cup Y_1) \ldots (X_{k-1} \cup Y_{k-1})$. Note that $\T$ generates a strategy in the way of restricting $\tau$ to return only one of its values (chosen arbitrarily).

\begin{lemma}\label{lem:synthesis}
    Let $\P$ be a problem of \LTLf synthesis with duties and rights, and $\T$ constructed as above. Any strategy returned by $\T$ is a strategy that solves the synthesis of $\P$.
\end{lemma}
\begin{proof}
Let $\stag$ be an arbitrary strategy generated by $\T$, i.e., $I \in \wina$, and $\stenv \in \llbracket env \rrbracket$ be an arbitrary environment strategy that enforces $env$. First, $\T$ already restricts the environment to be able to only choose strategies from $\llbracket env \rrbracket$. Then, by construction, $\play(\stenv, \stag)$ satisfies the following:
\begin{itemize}
    \item 
    $\play(\stenv, \stag) \models \dutytask$, since $\stag$ forces $\play(\stenv, \stag)$ to get closer to $R_{\duties}$ at every step until reaching $R_{\duties}$. Moreover, $\stag$ starts playing $\stopact$ only after then.
    \item
    there exists $\stag_{h} \ \rhd_h \ \dutytask \wedge \rightstask$ for every prefix $h$ of $\play(\stenv, \stag)$. This holds since $\stag$ restricts $\play(\stenv, \stag)$ to visit states in $\wina_{\rights}$ only. Therefore, for every prefix $h$ of $\play(\stenv, \stag)$,
    there exists an agent strategy $\hat{\stag}$ of the reachability game on $\A_h = (2^{\X \cup \Y}, Q, \delta(I, h), \delta, \reach(R_{\duties} \cap R_{\rights}))$. 
    Hence, we can construct an agent strategy $\stag_h$ that first copies $h$ until reaching $\delta(I, h)$ and works as $\hat{\stag}$ until reaching $R_{\duties} \cap R_{\rights}$, then plays $\stopact$ forever. Therefore, $\stag_h$ holds that $\stag_h \ \rhd_h \dutytask \wedge \rightstask$. \qedhere
\end{itemize}
\end{proof}

Notice that by the construction described above, if the reachability game on $\A$~(in Step 1) does not have an agent winning strategy, then $\T$ trivially returns \textit{no strategy} and indeed, by Lemma~\ref{lem:realizability}, $\P$ is unrealizable. As an immediate consequence of Lemmas~\ref{lem:realizability}\&\ref{lem:synthesis}, we have:
\begin{theorem}\label{thm:correctness}
    Let $\P$ be a problem of \LTLf synthesis with duties and rights. Realizability of $\P$ can be solved by reducing to a suitable reachability game. Synthesis of $\P$ can be solved by generating a strategy from $\T$ constructed as above. 
\end{theorem}
\begin{proof}
Immediate consequence of Lemmas~\ref{lem:realizability}\&\ref{lem:synthesis}.
\end{proof} 

\begin{theorem}\label{thm:realizability-complexity}
    Let $\P$ be a problem of \LTLf synthesis with duties and rights. Realizability of $\P$ is 2EXPTIME-complete.
\end{theorem}
\begin{proof}
We prove from the following two aspects.
\item{Membership.}
Constructing the automata from \LTLf formulas $\dutytask, \rightstask$ and $env$ contributes to the main computational complexity of solving $\P$, which takes 2EXP\-TIME. 
The final reachability game can be constructed in polynomial time in the size of the DAs $\A_{\env}$, $\A_{\duties}$ and $\A_{\rights}$. Solving the final reachability game can be done in linear time in the size of the game arena.

\item{Hardness.}
 Immediate from 2EXPTIME-completeness of \LTLf synthesis itself~\cite{DegVa15}.
\end{proof}

\begin{theorem}\label{thm:synthesis-complexity}
    Let $\P$ be a problem of \LTLf synthesis with duties and rights. Then computing a strategy solving $\P$ can take, in the worst case, double-exponential time in the size of $|\dutytask|+|\rightstask|+|env|$.
\end{theorem}
\begin{proof}
Immediate consequence of the construction of $\T$ and the membership proof of Theorem~\ref{thm:realizability-complexity}.
\end{proof} 

We observe that if $env$ is specified, for example, using, e.g., PDDL~\cite{2019Haslum} instead of \LTLf, then the complexity with respect to the environment specification $env$ only becomes EXPTIME-complete~(membership from a construction, hardness from planning of Fully Observable Nondeterministic Domains~(FOND)~\cite{Rintanen04a}).

\smallskip
\noindent\textbf{Enforcing also rights while executing.} Given problem $\P= (env, \dutytask, \rightstask)$, suppose we have synthesized a strategy $\stag$ for $\P$, which enforces $\dutytask$ and is right-aware for $\rightstask$, and while executing $\stag$, the agent wants to satisfy also its rights $\rightstask$. Then we can consider the history $h$ generated with the environment so far, and synthesize a strategy $\stag_h$, that enforces $\dutytask \wedge \rightstask$ with respect to history $h$. This can take 2EXPTIME, as shown by Theorem~\ref{thm:enforce-with-h}.

Nevertheless, if we consider the construction above, we actually do not need to compute the new strategy $\stag_h$ from scratch. This is because we can, base on the immediate results obtained from computing the original strategy $\stag$, to construct a transducer $\T_{\rights}$ for generating $\stag_h$ of a given history $h$. In particular, this transducer is independent of $h$. Therefore, we can construct $\T_{\rights}$ apriori, and use it to obtain $\stag_h$ when the agent chooses to satisfy $\rightstask$ after history $h$. The essential ingredients for constructing $\T_{\rights}$ is the DA $\A = (2^{\X \cup \Y}, Q, I, \delta, \reach(R_{\duties} \cap R_{\rights}))$ and the agent winning region $\wina_{\rights} = \bigcup_{0 \leq i \leq m} \wina_{\rights}^i$. We construct $\T_{\rights} = (2^{\X \cup \Y}, Q, I, \varrho_{\rights}, \tau_{\rights})$ as follows:
\begin{itemize}
    \item
    $2^{\X \cup \Y}$, $Q$ and $I$ are the same as in $\A$;
    \item
    $\varrho_{\rights}: Q \times 2^{\X} \rightarrow 2^Q$ is the transition function such that 
    $\varrho_{\rights} (q, X) = \{q'\mid q' = \delta (q, X \cup Y) \text{ and } Y \in \tau_{\rights}(q)\};$
	\item 
	$\tau_{\rights}: Q \times 2^\X \rightarrow 2^{2^{\Y}}$ is the output function such that $\forall X \in 2^\X$, $\tau_{\rights}(q, X) = \{ Y \mid \delta(q, X \cup Y) \in \wina_{\rights}^i \}$ if $q \in \wina_{\rights}^{i+1} \backslash \wina_{\rights}^i$, otherwise $\tau_{\rights}(q, X) = 2^\Y$.
\end{itemize}

Suppose while executing $\stag$, which enforces $\dutytask$ and is right-aware for $\rightstask$, the agent chooses to satisfy $\rightstask$ after history $h$, the transducer $\T_{\rights}$ generates an agent strategy $\stag_h: (2^\X)^+ \to 2^{\Y}$ in the following way: for every $\xi^k \in (2^{\X})^+$ that is compatible with history $h = (X_0 \cup Y_0) (X_1 \cup Y_1) \ldots (X_i \cup Y_i)$

\noindent$
\stag_h(\xi^k) = 
\begin{cases}
    Y_k & \text{ if } 0 \leq k \leq i \\
    \stopact & \text{ if } \run(\A, \pi^{k{-}1}) \text{ visited } R_{\duties} \cap R_{\rights} \\
    Y \in \tau(q_k, X_k) & \text{ otherwise. }
\end{cases}
$
where $\run(\A, \pi^{k-1}) = q_0q_1q_2 \ldots q_k$ such that $q_0 = I$, and $\pi^{k-1} = (X_0 \cup Y_0) (X_1 \cup Y_1) \ldots (X_{k-1} \cup Y_{k-1})$. Intuitively, given history $h$, $\T_{\rights}$ generates a strategy $\stag_h$ by first following $h$, and after $h$, choosing suitable agent action to enforce the play to get closer to $R_{\duties} \cap R_{\rights}$. At the end, $\stag_h$ keeps playing $\stopact$ right after visiting $R_{\duties} \cap R_{\rights}$.
 
\begin{theorem}
    Let $\P$ be a problem of \LTLf synthesis with duties and rights, $\stag$ be an agent strategy computed by $\T$ that solves the synthesis of $\P$, and $\stag_h$ be an agent strategy that is generated by $\T_{\rights}$ for a history $h \in (2^{\X \cup \Y})^*$. Then $\stag_h$ \emph{enforces} $\dutytask \wedge \rightstask$, with respect to history $h$, i.e., $\stag_h \ \rhd_h  \dutytask \wedge \rightstask$.
\end{theorem}
\begin{proof}
    Let $\stag_h$ be an arbitrary strategy generated by $\T_{\rights}$ for history $h$, and $\stenv \in \llbracket env \rrbracket$ be an arbitrary environment strategy that enforces $env$. By Lemma~\ref{lem:synthesis}, we have that $\run(\A, h)$ only visits states in $\wina_{\rights}$. By construction of $\stag_h$, $\play(\stenv, \stag)$ has $h$ as a prefix, and $\play(\stenv, \stag) \models \dutytask \wedge \rightstask$. Therefore, $\stag_h \ \rhd_h  \dutytask \wedge \rightstask$ holds.
\end{proof}

The advantage of building transducer $\T_{\rights}$ is that this transducer works for any history $h$ generated by $\stag$ that enforces $\dutytask$ and is right-aware for $\rightstask$. Moreover, when building the transducer $\T$ for $\stag$, we already have all the ingredients to build also $\T_{\rights}$, with only a constant overhead~(i.e., since we are computing two transducers, sharing essentially the same cost, instead of one). 

We now extend Example~\ref{example:1} to show how to utilize the transducer $\T_\rights$ in the presence of robot also achieving rights. 
\begin{example}
Suppose the robot decides to also achieve its rights $\rightstask = \Diamond (\textit{BatteryFull})$ while cleaning room A. Let us assume that the by now the running history is $h$. The robot will look into $\T_\rights$ and choose a strategy $\sigma_h$ out of $\T_\rights$ that allows it to enforce $\dutytask \wedge \rightstask$.
\end{example}

\section{Handling Further Duties and Rights While Executing}\label{sec:further-d-r}

Let us focus on duty only first. Commonly in synthesis the agent only gets one task (duty) to accomplish, after which, the agent can terminate.  However, in practice, further tasks might arrive while executing the current task, e.g., a new room to clean while the robot is cleaning the rooms it got assigned at the beginning. Intuitively, the new task can be considered as an update of the previous task.

Synthesizing updated specifications has been recently studied in Formal Methods, under the name of \emph{live synthesis}~\cite{livesynthesis}, where the desired properties are specified in \LTL and can get updated while executing a strategy of the original \LTL specification. The goal of live synthesis is to synthesize a new strategy to replace an already running strategy. In particular, the correct handover from the already running strategy to the new strategy is specified by an extension of \LTL, called Live\LTL. For specifications in Live\LTL, the synthesis problem shares the same complexity bound as standard \LTL synthesis. 

Despite that synthesis problems of \LTLf can be solved by a reduction to suitable problems of \LTL, since \LTLf can be encoded in \LTL, such reductions do not seem promising, as shown in~\cite{ZTLPV17} for \LTLf synthesis, and~\cite{ZhuGPV20,GiacomoSVZ20} for \LTLf synthesis under environment specifications. 
So while we want to consider an \LTLf variant of live synthesis, we avoid a detour to Live\LTL synthesis and devise a direct synthesis technique for \LTLf. 

We start by observing that the crucial difference between new duties and the ongoing duties is that, the agent should enforce ongoing duties from the very beginning, but enforce new duties after a history, i.e., starting from the moment that the new duties are assigned to the agent.
\begin{definition}\label{def:strategy-after-h}
    Let $\varphi$ be an \LTLf formula and $h \in (2^{\X \cup \Y})^*$ be a history. An agent strategy $\stag$ \emph{enforces} $\varphi$ after history $h$, denoted by $\stag \ \rhd_\afh \ \varphi$, if $\forall \stenv \in \llbracket env \rrbracket$ such that $\play(\stenv, \stag)$ has $h$ as a prefix, we have that $\play(\stenv, \stag), |h| \models \varphi$.
\end{definition}
Recall that Definition~\ref{def:strategy-with-h} describes how an \emph{agent strategy enforces $\varphi$ with respect to a history $h$}, and Definition~\ref{def:strategy-after-h} above describes how an \emph{agent strategy enforces $\varphi$ after a history $h$}.
There is a significant difference between these two notions, since we use them to differentiate how agent strategies enforce ongoing duties and new duties. In particular, an agent strategy enforces ongoing duties with respect to a history $h$, but enforces new duties after a history $h$.

Note that the environment strategy enforcing $env$ in any case starts from the very beginning. Moreover, we only consider the cases where both the environment strategy $\stenv$ and the agent strategy $\stag$ are compatible with $h$. In other words, we need to consider environment strategies that are in the set
$\llbracket env \rrbracket^h = \{ \stenv \mid \forall \stag$ that is compatible with $h$ we have $\trace(\stenv, \stag)$ has $h$ as a prefix and $\trace(\stenv, \stag) \models_\forall env\}$.

In order to compute a strategy $\stag$ that enforces $\varphi$ after $h$, we split the trace $\trace(\stenv, \stag)$ into two phases. 
In phase \rom{1}, both strategies $\stenv$ and $\stag$ are compatible with $h$. 
In phase \rom{2}, the agent focuses on the environment strategies that enforce $env$ with respect to $h$. We show how to compute a strategy $\stag$ enforcing $\varphi$ after $h$ by addressing two phases in reverse order. 
Specifically, we first compute the set of environment strategies that start executing after $h$, but enforce $env$ when the compatible traces are concatenated to $h$. We denote this set of environment strategies by 
{\centerline{$\llbracket env, \afh \rrbracket = \{ \stenv \mid \forall \stag.~h \cdot \trace(\stenv, \stag) \models_\forall env\}.$}}

The fact that we can focus on this set of environment strategies is justified by the following lemma, which is easy to prove considering the two definitions of $\llbracket env \rrbracket^h$ and $\llbracket env,\afh \rrbracket$.
\begin{lemma}\label{lem:env-after-h}
    For every $\stenv\in \llbracket env\rrbracket^h$ there exists $\stenv'\in \llbracket env, \afh \rrbracket$ s.t. for every $\lambda= h|_\Y\cdot \lambda'$\, we have 
    $\stenv(\lambda) =  \stenv'(\lambda')$.
    
    Viceversa, for every $\stenv'\in \llbracket env, \afh \rrbracket$ there exists $\stenv\in \llbracket env \rrbracket^h$ s.t. for every $\lambda= h|_\Y\cdot \lambda'$\, we have 
    $\stenv(\lambda) =  \stenv'(\lambda')$.  
\end{lemma}

\begin{proof}
By contradiction, considering the two definitions of $\llbracket env\rrbracket^h$ and $\llbracket env, \afh \rrbracket$.
Let us consider the first direction (the other direction is similar). 

Suppose that there exists a strategy $\stenv\in \llbracket env, \afh \rrbracket$, and for every strategy $\hat{\stenv}\in \llbracket env \rrbracket^h$, we have that for $h|_\Y\cdot \lambda$, $\stenv(\lambda) \neq  \hat{\stenv}(h|_\Y\cdot \lambda)$. Let's consider any $\stenv' \in \llbracket env \rrbracket^h$ such that for all the prefixes $l$ of $\lambda$, it holds that $\stenv(l) =  \stenv'(h|_\Y \cdot l)$, and then for $\lambda$, it holds that $\stenv(\lambda) \neq  \stenv'(h|_\Y \cdot \lambda)$.
Note that every trace $\pi$, that is compatible with $h \cdot \stenv$, holds that $\pi \models_\forall env$. Now consider the strategy $\stenv''$ defined as follows: for every $\lambda'' \in (2^\Y)^*$,

$\stenv''(\lambda'')=
\begin{cases}
\stenv'(\lambda'')  &\text{ if } \lambda'' \text{ does not have } h|_\Y \text{ as a prefix }\\
\stenv'(\lambda'')  &\text{ if } \lambda'' \text{ is a proper prefix of } h|_\Y \cdot \lambda \\
\stenv(\iota)  &\text{ if } \lambda''= h|_\Y \cdot \iota \text{ and } \lambda \text{ is a prefix of } \iota
\end{cases}
$
It is immediate to see that $\stenv''$ belongs to the set $\llbracket env \rrbracket^h$, leading to a contradiction.
\end{proof}

To compute  $\llbracket env, \afh \rrbracket$, we first compute $\llbracket env \rrbracket$, represented as $\A'_\env = (2^{\X \cup \Y}, \wine, I_\env, \delta'_\env, \safe(\wine))$, as described in Section~\ref{sec:synthesis-rights}. In order to synchronize the starting point of the environment to be aligned with the instant after history $h$, we run $\A'_\env$ on $h$ to obtain a new \DA $\A'_{\env,\afh} = (2^{\X \cup \Y}, \wine, I_{\env,\afh}, \delta'_\env, \safe(\wine))$ that differs from $\A'_\env$ only on the initial state, and $I_{\env,\afh} = \delta'_\env(I_\env, h)$. The following lemma shows that $\A'_{\env,\afh}$ precisely captures $\llbracket env, \afh \rrbracket$, which is easy to prove by construction.

\begin{lemma}
    Let $env$ be an \LTLf formula specifying a safety property, $h \in (2^{\X \cup \Y})^*$ be a finite history, and $\A'_\env$ be constructed as above. Then $\A'_{\env,\afh}$ represents the set $\llbracket env, \afh \rrbracket$.
\end{lemma}
\begin{proof}
Note that $\A'_{\env,\afh}$ can actually be considered as a representation of the following transducer $\T_{\env,\afh} = (2^{\X \cup \Y}, \wine, I_{\env,\afh}, \varrho_{\env,\afh}, \tau_{\env,\afh})$ that encodes a set of environment strategies, where
\begin{compactitem}
    \item
    $2^{\X \cup \Y}$, $\wine$ and $I_{\env,\afh}$ are the same as in $\A'_{\env,\afh}$;
    \item
    $\varrho_{\env,\afh}: \wine \times 2^{\Y} \mapsto 2^{\wine}$ is the transition function s.t. 
    $\varrho_{\env,\afh} (q, Y) = \{q' \mid q' = \delta'_\env (q, X \cup Y) \text{ and } X \in \tau_{\env,\afh}(q)\}$;
	\item 
	$\tau_{\env,\afh}: \wine \to 2^{2^{\X}}$ is the output function such that $\tau_{\env,\afh}(q) = \{ X \mid \text{ if } \forall Y \in 2^\Y. \delta'_\env(q, X \cup Y) \in \wine \}$.
\end{compactitem}
$\T_{\env,\afh}$ generates an environment strategy $\stenv: (2^\Y)^* \to 2^{\X}$ in the following way: $\stenv(\epsilon) = X \in \tau_{\env,\afh}(I_{\env,\afh})$, and for every $\lambda^k \in (2^{\Y})^+, \stag(\lambda^k) = X_{k+1} \in \tau_{\env,\afh}(q_{k+1})$, where $q_{k+1}$ indicates the last state of $\run(\A, \pi^{k}) = q_0q_1q_2 \ldots q_{k+1}$ such that $q_0 = I_{\env,\afh}$, and $\pi^{k} = (X_0 \cup Y_0) (X_1 \cup Y_1) \ldots (X_{k} \cup Y_{k})$. We denote the set of environment strategies that $\T_{\env,\afh}$ can generate by $\llbracket\T_{\env,\afh}\rrbracket$, and now prove the lemma by showing $\llbracket\T_{\env,\afh}\rrbracket = \llbracket env, \afh \rrbracket$.

\item{($\Leftarrow$)}
Let $\stenv$ be an arbitrary environment strategy such that $\stenv \in \llbracket\T_{\env,\afh}\rrbracket$. Note that $\A'_{\env,\afh}$ only differs from $\A'_\env$ on the initial states, therefore, by construction, for every agent strategy $\stag$, the run $\run(\A_\env, \pi) = q_0 q_1 q_2 \ldots \in Q^\omega$ of the induced trace $\pi = \trace(\stenv, \stag)$ satisfies that $q_0 = I_{\env,\afh}$ and $\forall k \geq 0: q_k \in \wine$. Moreover, since $\run(\A'_\env, h)$ is guaranteed to visit $\wine$ only, it holds that $h \cdot \pi \models_\forall env$, and therefore $\stenv \in \llbracket env, \afh \rrbracket$.

\item{($\Rightarrow$)}
Let $\stenv$ be an arbitrary environment strategy s.t. $\stenv \in \llbracket env, \afh \rrbracket$. By definition, for every agent strategy $\stag$, $\pi = \trace(\stenv, \stag)$ holds that $h \cdot \pi \models_\forall env$ such that $\rho = \run(\A'_\env, h \cdot \pi)$ visits states in $S$ only~(recall that $\A_\env = (2^{\X \cup \Y}, Q_\env, I_\env, \delta_\env, \safe(S))$ is the \DA of $env$). 
Since it is guaranteed that $\run(\A'_\env, h)$ visits states in $\wine$ only, in order to make $h \cdot \trace(\stenv, \stag)$ is winning for the environment for the safety game on $\A_\env$, $\stenv$ should be able to enforce the trace to stay in $\wine$. Otherwise, $\stenv$ cannot be a winning strategy for the environment. By construction, $\T_{\env, \afh}$ captures all such environment strategies, and therefore $\stenv \in \llbracket\T_{\env,\afh}\rrbracket$.
\end{proof}

Having $\llbracket env, \afh \rrbracket$ represented as DA $\A'_{\env,\afh}$, we can first construct the DA $\A_\varphi$ and then solve a reachability game on the product $\A'_{\env,\afh} \times \A_\varphi$ to abstract an agent strategy $\hat{\stag}$ that guides the play to satisfy the reachability condition of $\A_\varphi$, hence enforcing $\varphi$. The final agent strategy enforcing $\varphi$ after $h$ can be obtained by first copying $h$, and then switching to $\hat{\stag}$ after $h$. Formally, for every $\xi^k \in (2^\X)^+$ that is compatible with $h = (X_0 \cup Y_0) (X_1 \cup Y_1) \dots (X_i \cup Y_i) \in (2^{\X \cup \Y})^*$

$\stag(\xi^k) = 
\begin{cases}
    Y_k & \text{ if } 0 \leq k \leq i, \\
    \hat{\stag}(\iota) & \text{ if } \xi^k = h \cdot \iota.
\end{cases}
$

\begin{lemma}\label{lem:strategy-enforce-after-h}
    Let $\varphi$ be an \LTLf formula, $h \in (2^{\X \cup \Y})^*$ be a  history, and $\stag$ be constructed as above. Then $\stag$ enforces $\varphi$ after $h$.
\end{lemma}
\begin{proof}
Note that $\stag$ is constructed from $h$ and a strategy $\hat{\stag}$, which holds that for every $\hat{\stenv} \in \llbracket env, \afh \rrbracket$, $\play(\hat{\stenv}, \hat{\stag}) \models \varphi$. Hence, together with Lemma~\ref{lem:env-after-h}, it holds that for every $\stenv' \in \llbracket env \rrbracket^h$, $\play(\stenv', \stag), |h| \models \varphi$. Clearly, $\stag$ holds that for every $\stenv \in \llbracket env \rrbracket$ such that $\play(\stenv', \stag)$ has $h$ as a prefix, then $\play(\stenv, \stag), |h| \models \varphi$. Therefore, $\stag$ enforces $\varphi$ after $h$.
\end{proof}

The following theorem shows that computing an agent strategy enforcing $\varphi$ after a history $h$ is not more difficult than computing a strategy that just enforces $\varphi$.

\begin{theorem}\label{thm:enforce-after-h}
    Let $\varphi$ be an \LTLf formula and $h \in (2^{\X \cup \Y})^*$ be a history. Computing an agent strategy that enforces $\varphi$ after history $h$ is 2EXPTIME-complete in $\varphi$.
\end{theorem}
\begin{proof}
  We prove from the following two aspects.
    \item{Membership.}
    Constructing the automata from \LTLf formulas $\varphi$ and $env$ contributes to the main computational complexity of computing  such a strategy, which takes 2EXP\-TIME. 
    The final reachability game can be constructed in polynomial time in the size of the DAs $\A_{\env}$, $\A_{\varphi}$, consisting of safety game solving on $\A_{\env}$, restricting to $\wine$, and automata product. Solving the reachability game can be done in linear time in the size of the game arena.
  \item{Hardness.}
  Immediate from 2EXPTIME-completeness of \LTLf synthesis itself~\cite{DegVa15}.
\end{proof}

Building on the above results and the results in Section~\ref{sec:synthesis-rights}, we now enrich our synthesis setting by allowing both further duties and further rights, specified as \LTLf formulas $\fdutytask$ and $\frightstask$, respectively. 

\begin{definition}
Let $h$ be the formed history when further duties $\fdutytask$ and rights $\frightstask$ arrive. Agent strategy $\stag$ enforcing $\fdutytask$ is right-aware for $\frightstask$ after $h$, if $\forall \stenv \in \llbracket env \rrbracket$:
\begin{itemize}
    \item 
    $\play(\stenv, \stag)$ has $h$ as a prefix and $\play(\stenv, \stag), |h| \models \fdutytask$;
    \item
    for every prefix $l$ of $\play(\stenv, \stag)$ that has $h$ as a prefix, there exists an agent strategy $\stag_{l}$ that enforces $\fdutytask \wedge \frightstask$ after history $h$, i.e., $\stag_{l} \ \rhd_\afh \ \fdutytask \wedge \frightstask$.
    \end{itemize}
\end{definition}
The enriched synthesis problem that allows further duties and rights is defined as follows.
\begin{definition}[\LTLf synthesis for further duties and rights]\label{def:enriched-synthesis}
The problem is described as a tuple $\hat{\P}= (env, \dutytask, \rightstask, \fdutytask, \frightstask, h)$, where $env$ is an \LTLf formula specifying the environment safety specification, $\dutytask$ and $\rightstask$ are \LTLf formulas specifying duties and rights, respectively, $\fdutytask$ and $\frightstask$ are \LTLf formulas specifying further duties and rights that arrive after history $h$, respectively.
Realizability of $\hat{\P}$ checks whether there exists an agent strategy $\stag$ s.t.:
\begin{itemize}
    \item it enforces $\dutytask$ and is right-aware for $\rightstask$ wrt $h$;
    \item it enforces $\fdutytask$ and is right-aware for $\frightstask$ after $h$.
\end{itemize}
Synthesis of $\hat{\P}$ computes a strategy $\stag$ if exists.
\end{definition}  
Notice that, in Definition~\ref{def:enriched-synthesis}, we also only consider the cases where both the environment strategies and the agent strategies are compatible with $h$. 

We extend Example~\ref{example:1} to address the intuition of this class of synthesis problem.
\begin{example}
Suppose the robot is on its way to room A, while receiving a new duty of ``cleaning room B" $\fdutytask = \Diamond (\neg \textit{Dust\_B} \wedge \textit{RobotOut\_B})$. The robot has generated a history $h$ when receiving $\fdutytask$. Now, the robot has one strategy to enforce $\dutytask$ with respect to $h$ and enforce $\fdutytask$ after $h$: 
\begin{compactitem}
    \item
    Take the direction that passes the charging station to reach room A and clean it. Then go to room B and clean it. The remaining battery after cleaning is not enough for the robot to reach the charging station.
\end{compactitem}
In this case, the robot would refuse the new duty $\fdutytask$, since completing it would be conflicted with maintaining the robot rights  $\rightstask$. 
\end{example}

On the other hand, let us consider the situation where the cleaning robot is able to handle original duties and rights, together with further duties and rights. 
\begin{example}
Suppose the cleaning robot is cleaning room A, while receiving a new duty of ``cleaning room C" $\fdutytask = \Diamond (\neg \textit{Dust\_C} \wedge \textit{RobotOut\_C})$, and a new right of ``emptying the garbage collector" $\frightstask = \Diamond (\textit{Collector\_Empty})$. The robot has generated a history $h$ when receiving $\fdutytask$ and $\frightstask$. 
Since this is a circular hallway, the robot again has two directions to reach room C. Nevertheless, the robot only takes the strategy that allows it to reach the charging station, and the garbage station whenever it wants.
\end{example}

\subsection{Synthesis Technique}
Given problem $\hP= (env, \dutytask, \rightstask, \fdutytask, \frightstask, h)$, the main complication comes from further duties and rights that arrive after history $h$. This is because apart from enforcing $\fdutytask$ while maintaining $\frightstask$, the agent should also enforce unfinished $\dutytask$ and be right-aware for $\rightstask$. To synthesize an agent strategy that is able to do so, we do the following: 
\myi compute the agent winning region $\wina_{\rights}$ from where the agent is able to lead the trace to satisfy $\dutytask$ and $\rightstask$;
\myii compute the agent winning region $\wina_{\rights \wedge \fr}$ from where the agent is able to lead the trace to also satisfy both $\fdutytask$ and $\frightstask$, but after $h$;
\myiii synthesize an agent strategy $\stag$ enforcing $\dutytask$ and is right-aware for $\rightstask$ wrt $h$, also enforcing $\fdutytask$ being right-aware for $\frightstask$, but after $h$.
We now elaborate on every step.

\smallskip
\smallskip
\noindent\textbf{Step 1.} Compute $\wina_\rights$. As described in Section~\ref{sec:3-1}, we can construct the corresponding DAs $\A_{\duties} = (2^{\X \cup \Y}, Q_{\duties}, I_{\duties}, \delta_{\duties}, \reach(R_{\duties}))$ and $\A_{\rights} = (2^{\X \cup \Y}, Q_{\rights}, I_{\rights}, \delta_{\rights}, \reach(R_{\rights}))$ of $\dutytask$ and $\rightstask$, respectively. The agent winning region $\wina_\rights$ can be computed via a least fixpoint computation on the product DA $\A_{\duties \wedge \rights } = (2^{\X \cup \Y}, Q_{\duties \wedge \rights }, I_{\duties \wedge \rights}, \delta_{\duties \wedge \rights}, \reach(R)_{\duties \wedge \rights})$ constructed out of $\A'_\env$, $\A_{\duties}$, and $\A_{\rights}$, where $\A'_\env$ captures $\llbracket env \rrbracket$. 

\smallskip
\smallskip
\noindent\textbf{Step 2.} Compute $\wina_{\rights \wedge \fr}$. Build \DA $\A_\fd = (2^{\X \cup \Y}, Q_\fd, I_\fd, $ $\delta_\fd, \reach(R_\fd))$ of $\fdutytask$, and \DA $\A_\fr = (2^{\X \cup \Y}, Q_\fr, I_\fr, \delta_\fr, \reach(R_\fr))$ of $\frightstask$. 
In order to synthesize an agent strategy that enforces $\fdutytask$ and is right-aware $\frightstask$ after $h$, we need to run $\A_{\duties \wedge \rights}$ on $h$ to obtain a new \DA $\A'_{\duties \wedge \rights}$ that differs from $\A_{\duties \wedge \rights}$ only on the initial state, and $\A_{\duties \wedge \rights, \afh} = (2^{\X \cup \Y}, Q_{\duties \wedge \rights}, I_{\duties \wedge \rights,\afh}, \delta_{\duties \wedge \rights}, \reach(R))$, where $I_{\duties \wedge \rights, \afh} = \delta(I_{\duties \wedge \rights}, h)$. 
Clearly, if $\run(\A_{\duties \wedge \rights}, h)$ does not visit states in $\wina_{\rights}$ only, return ``unrealizable", since every agent strategy that is compatible with $h$ cannot enforce $\dutytask$ and be right-aware for $\rightstask$, thus $\hP$ is simply unrealizable, and so $\wina_{\rights \wedge \fr} = \emptyset$. 

Otherwise, we continue as follows. 
Since the final agent strategy should be able to guide the play to reach $R_{\duties}$ and $R_\fd$, and always able to reach also $R_\rights$ and $R_\fr$, we take the product of $\A_{\duties \wedge \rights, \afh}$, $\A_\fd$ and $\A_\fr$ into
$\A = (2^{\X \cup \Y}, Q, I, \delta, \reach(R))$, where $Q = Q_{\duties \wedge \rights} \times Q_\fd \times Q_\fr$, $I = (I_{\duties \wedge \rights, \afh}, I_\fd, I_\fr)$, $\delta((q_1,q_2,q_3), X\cup Y) = (\delta_{\duties \wedge \rights}(q_1,X \cup Y), \delta_\fd(q_2,X \cup Y), \delta_\fr(q_3,X \cup Y))$.
Furthermore, $\delta((q_1,q_2,q_3), X\cup Y) = \text{ undefined}$, if $\delta_{\duties \wedge \rights}(q_1,X \cup Y) = \text{ undefined}$. 
Finally, $R = R_{\duties} \cap R_\rights \cap R_\fd \cap R_\fr$. 
We now solve a reachability game on $\A$ for the agent via a least fixpoint computation, to obtain $\wina_{\rights \wedge \fr} = \bigcup_{0 \leq i \leq m} \wina_{\rights \wedge \fr}^i$. If $I \not\in \wina_{\rights \wedge \fr}$, i.e., $\A$ does not have an agent winning strategy, return ``unrealizable".

\begin{lemma}\label{lem:enriched-realizability}
Let $\hat{\P}$ be a problem of \LTLf synthesis for further duties and rights, and $\wina_{\rights \wedge \fr}$ the agent winning region of the reachability game on $\A = (2^{\X \cup \Y}, Q, I, \delta, \reach(R))$ computed as above. $\hat{\P}$ is realizable iff $I \in \wina_{\rights \wedge \fr}$.
\end{lemma}
\begin{proof}
We prove the lemma in both directions.
\item{($\Leftarrow$)} We need to show that if $I \in \wina_{\rights \wedge \fr}$, then $\hP$ is realizable.
By construction, $I \in \wina_{\rights \wedge \fr}$ shows that there exists an agent strategy $\stag$ such that, for every $\stenv \in \llbracket env, \afh \rrbracket$, $\pi = \trace(\stenv, \stag)$ is such that $\pi \in \L(\A)$.
That is to say, $\pi^k \models \fdutytask \wedge \frightstask$ for some $k \geq 0$, thus it also holds that $\pi^k \models \fdutytask$. Moreover, since $I = (I_{\duties \wedge \rights, \afh}, I_\fd, I_\fr)$, where $I_{\duties \wedge \rights, \afh} = \delta_{\duties \wedge \rights}(I, h)$, it also holds that $h \cdot \pi^k \models \dutytask \wedge \rightstask$, and thus $h \cdot \pi^k \models \dutytask$. Moreover, for every prefix $l$ of $\play(\stenv, \stag)$ that has $h$ as a prefix, we can construct an agent strategy $\stag_l$ that works exactly the same as $\stag$, which indeed enforces $\rightstask$ with respect to $h$, and enforces $\frightstask$ after $h$.

\item{($\Rightarrow$)} We prove by contradiction. If $I \not\in \wina_{\rights \wedge \fr}$, then either $\wina_{\rights \wedge \fr} = \emptyset$, i.e., $\run(\A_{\duties \wedge \rights}, h)$ does not visit states in $\wina_{\rights}$ only, then $\hP$ is simply unrealizable; or there does not exist an agent winning strategy of reachability game on $\A$. Therefore, suppose the agent decides to achieve both rights $\rightstask$ and $\frightstask$ immediately after history $h$, then the agent does not have a strategy that enforces $\dutytask \wedge \rightstask$ with respect to history $h$, and $\fdutytask \wedge \frightstask$ after $h$. Hence, $\hP$ is unrealizable.
\end{proof}

\noindent\textbf{Step 3.} Compute strategy $\stag$. Note that strategy $\stag$ needs to lead the play to reach $R_{\duties}$ and $R_\fd$ by visiting states in $\wina_{\rights\wedge \fr}$ only. Moreover, if the agent already decides to achieve $\rightstask$ along $h$, here the strategy $\stag$ should lead the play to reach, instead, $R_{\duties} \cap R_\fd \cap R_\rights$, by visiting states in $\wina_{\rights\wedge \fr}$ only. The following computation focuses on the former case~(computing a strategy for the latter case is similar).
Therefore, we can build a new \DA with reachability-safety condition $\A_1 = (2^{\X \cup \Y}, Q, I, \delta, \reachsafe(R_{\duties} \cap R_\fd, \wina_{\rights \wedge \fr}))$ out of $\A = (2^{\X \cup \Y}, Q, I, \delta, \reach(R))$, and solve it by reducing to a reachability game, which is analogous to the construction presented in Section~\ref{sec:3-1}. 

Then we solve the reduced reachability game via a least fixpoint computation and obtain $\wina = \bigcup_{0 \leq j \leq n} \wina^j$. Note that $I \in \wina$ indeed holds, which is guaranteed by the reachability game for computing $\wina_{\rights \wedge \fr}$ in the previous step. We now define a strategy generator that starts serving after $h$, based on 
$\wina = \bigcup_{0 \leq j \leq n} \wina^j$, 
represented as a transducer 
$\hat{\T} = (2^{\X \cup \Y}, Q, I, \varrho, \tau)$, where
\begin{itemize}
    \item
    $2^{\X \cup \Y}$, $Q$ and $I$ are the same as in $\A$;
    \item
    $\varrho: Q \times 2^{\X} \rightarrow 2^Q$ is the transition function s.t.
    $\varrho (q, X) = \{q' \mid q' = \delta(q, X \cup Y) \text{ and } Y \in \tau(q)\};$
	\item 
	$\tau: Q \times 2^\X \rightarrow 2^{2^{\Y}}$ is the output function s.t.
	$\forall X \in 2^\X$, $\tau(q, X) = \{ Y \mid \delta(q, X \cup Y) \in \wina^{j} \}$ if $q \in (\wina^{j+1} \backslash \wina^{j})$, otherwise $\tau(q, X) = 2^\Y$.
\end{itemize}

The transducer $\hat{\T}$, together with history $h$, generates an agent strategy $\stag: (2^\X)^+ \to 2^{\Y}$ as follows: $\forall \xi^k \in (2^{\X})^+$ that is compatible with $h = (X_0 \cup Y_0) (X_1 \cup Y_1) \dots (X_i \cup Y_i) \in (2^{\X \cup \Y})^*$,

\noindent$\stag(\xi^k) = 
\begin{cases}
    Y_k & \text{ if } 0 \leq k \leq i, \\
    \stopact & \text{ if } \run(\A, \pi^{k-1}) \text{ visited } R_{\duties} \cap R_\fd, \\
    \tau(q_k, X_k) & \text{ otherwise. }
\end{cases}
$
where $\run(\A, \pi^{k-1}) = q_0q_1q_2 \ldots q_k$ such that $q_0 = I$, and $\pi^{k-1} = (X_0 \cup Y_0) (X_1 \cup Y_1) \ldots (X_{k-1} \cup Y_{k-1})$. Intuitively, given a history $h \in (2^{\X \cup \Y})^*$, the final strategy $\stag$ is generated by first following $h$, and starts taking $\hat{\T}$ by choosing suitable agent action to enforce the play to get closer to $R_{\duties} \cap R_\fd$, then keeps playing $\stopact$ right after visiting $R_{\duties} \cap R_\fd$.

\begin{lemma}\label{lem:enriched-synthesis}
    Let $\hat{\P}$ be a problem of \LTLf synthesis for further duties and rights, and $\hat{\T}$ be constructed as above. Any strategy returned by $\hat{\T}$ solves the synthesis of $\hat{\P}$.
\end{lemma}
\begin{proof}
Let $\stag$ be an arbitrary strategy generated by $\hat{\T}$, i.e., $I \in \wina_{\rights \wedge \fr}$, and $\stenv \in \llbracket env \rrbracket$ be an arbitrary environment strategy that enforces $env$ such that $\play(\stenv, \stag)$ has $h$ as a prefix. By construction, $\play(\stenv, \stag)$ satisfies the following:
\begin{itemize}
    \item 
    $\play(\stenv, \stag) \models \dutytask$ and $\play(\stenv, \stag), |h| \models \fdutytask$, since after $h$, $\stag$ forces $\play(\stenv, \stag)$ to get closer to $R_{\duties} \cap R_\fd$ at every step until reaching $R_{\duties} \cap R_\fd$. Moreover, only after reaching $R_{\duties} \cap R_\fd$, $\stag$ starts playing $\stopact$.
    \item
    for every prefix $l$ of $\play(\stenv, \stag)$, if $l$ is a prefix of $h$, then by definition, there exists an agent strategy $\stag_{l}$ that enforces $\dutytask \wedge \rightstask$ with respect to history $l$, i.e., $\stag_{l} \ \rhd_l \ \dutytask \wedge \rightstask$. Otherwise, $l$ has $h$ as a prefix, then by construction, after $h$, $\play(\stenv, \stag)$ only visits states in $\wina_{\rights \wedge \fr}$. Therefore, for any of the following cases, there exists an expected agent strategy that satisfies the conditions:
    \begin{itemize}
        \item
        enforcing $\dutytask \wedge \rightstask$ with respect to $h$, and $\fdutytask$ after $h$, if the agent chooses to achieve $\rightstask$ only;
        \item
        enforcing $\dutytask$ with respect to $h$, and $\fdutytask \wedge \frightstask$ after $h$, if the agent chooses to achieve $\frightstask$ only;
        \item
        enforcing $\dutytask \wedge \rightstask$ with respect to $h$, and $\fdutytask \wedge \frightstask$ after $h$, if the agent chooses to achieve both $\rightstask$ and $\frightstask$.\qedhere
    \end{itemize}
\end{itemize}
\end{proof}

By the construction described above, if the reachability game on $\A$ does not have an agent winning strategy, then $\hat{\T}$ trivially returns \textit{no strategy}, and indeed by Lemma~\ref{lem:enriched-realizability}, $\hat{\P}$ is unrealizable. 
%
As an immediate consequence of Lemmas~\ref{lem:enriched-realizability}\&\ref{lem:enriched-synthesis}, we have:
\begin{theorem}\label{thm:correctness}
    Let $\hat{\P}$ be a problem of \LTLf synthesis for further duties and rights. Realizability of $\hP$ can be solved by a reduction to a suitable reachability game. Synthesis of $\hP$ can be solved by generating a strategy from $\hat{\T}$, as constructed above. 
\end{theorem}
\begin{proof}
Immediate consequence of Lemmas~\ref{lem:enriched-realizability}\&\ref{lem:enriched-synthesis}.
\end{proof} 

\begin{theorem}\label{thm:enriched-realizability-complexity}
    Let $\hat{\P}$ be a problem of \LTLf synthesis for further duties and rights. Realizability of $\hP$ is 2EXPTIME-complete.
\end{theorem}
\begin{proof}
We prove from the following two aspects.
\item{Membership.}
Constructing the automata from \LTLf formulas $env, \dutytask, \rightstask, \fdutytask$ and $\frightstask$ takes 2EXPTIME. 
The final reachability game can be constructed in polynomial time in the size of the DAs $\A_{e}$, $\A_{\duties}$, $\A_{\rights}$, $\A_\fd$, and $\A_\fr$. Solving the final reachability game can be done in linear time in the size of the game arena.

\item{Hardness.}
 Immediate from 2EXPTIME-completeness of \LTLf synthesis itself~\cite{DegVa15}.
\end{proof} 

\begin{theorem}\label{thm:enriched-synthesis-complexity}
    Let $\hat{\P}$ be a problem of \LTLf synthesis for further duties and rights. Then computing a strategy solving $\hat{\P}$ can take, in the worst case, double-exponential time in $|env|+|\dutytask|+|\rightstask|+|\fdutytask|+|\frightstask|$.
\end{theorem}
\begin{proof}
Immediate consequence of the construction of $\hat{\T}$ and Theorem~\ref{thm:enriched-realizability-complexity}.
\end{proof} 

Analogously to Section~\ref{sec:3-1}, given problem $\hP= (env, \dutytask, \rightstask, \fdutytask, \frightstask, h)$, suppose we have a strategy $\stag$ for $\hP$, and while executing $\stag$, the agent wants to satisfy also its rights $\rightstask$ or $\frightstask$~(maybe both). Then we can again, base on the immediate results obtained from the construction above, to construct three transducers $\hat{\T}_{\rights}$, $\hat{\T}_{\fr}$ and $\hat{\T}_{\rights \wedge \fr}$, corresponding to agent options of achieving also $\rightstask$ only, $\frightstask$ only, and $\rightstask$ together with $\frightstask$, following the construction described in Section~\ref{sec:3-1}. Indeed, if the agent already decides to satisfy $\rightstask$ along $h$, i.e., before further duties and rights arrive, then the agent needs to stay with its choice of satisfying $\rightstask$. In this case, when further duties and rights arrive, the agent can only choose whether to satisfy also further rights $\frightstask$. 

We extend the cleaning robot example and demonstrate how the robot deal with further duties and rights.
\begin{example}
Suppose the cleaning robot decides to achieve also its new rights ``emptying the garbage collector" $\frightstask = \Diamond (\textit{Collector\_Empty})$ when being on its way to the charging station after cleaning rooms A and C. In other words, the robot already made the decision of achieving its rights ``fully charging battery", so the current strategy indeed comes from $\hat{\T}_{\rights}$. Let us assume that $\hat{\T}_{\rights \wedge \fr}$ is not prepared in advance and by now the running history is $h'$. Note that $h'$ is different from the running history $h$ generated when the robot receives further duties and rights. The robot will now compute $\hat{\T}_{\rights \wedge \fr}$, and choose a strategy $\sigma_{h'}$ out of $\hat{\T}_{\rights \wedge \fr}$ that allows it to enforce $\dutytask \wedge \rightstask$ with respect to $h$, and $\fdutytask \wedge \frightstask$ after $h$.
\end{example}

It should be noted that, one can generalize the problem of \LTLf synthesis with further duties and rights to allow multiple further duties and rights. In this case, computing such transducers in advance might lead to a tradeoff, since there can be an exponential number of agent options of choosing which rights to achieve. Therefore, it might be better to just keep the winning regions and compute the strategy for achieving the chosen rights only if and when the agent demands it.

\section{Conclusion}

We have studied synthesis for duties respecting rights. We have shown that we can actually compute such strategies with a small overhead wrt to the state-of-the-art \LTLf synthesis techniques~\cite{ZTLPV17,BLTV,GiacomoF21}.
We can do so by enriching the arena to contain also the information needed to handle the rights.
%
For simplicity, we have considered a single duties specification and a single rights specification at the time. Considering multiple duties specifications simultaneously actually is like considering as duties the conjunction of the duties specifications.
However, considering multiple rights specifications would require to consider satisfying arbitrary subsets of rights, as chosen by the agent. This can still be done with our techniques, though precomputing solutions as in Section~\ref{sec:synthesis-rights} can lead to a combinatorial explosion. In fact, our solution to handle further duties and rights in Section~4 can be already applied to handle multiple duties and rights as well, by considering as history  the empty history. Note that the technique presented there also handle contradicting rights, i.e., rights that cannot actually be satisfied simultaneously (but see the discussion at the end of Section~4). 
These extensions tighten up even more the connection with Deontic Logic, in particular in combination with actions~\cite{GabbayForthcoming-GABHOD-2}. We leave though exploring this connection to future work.

\section*{Acknowledgements}
This work is partially supported by the ERC Advanced Grant WhiteMech (No. 834228), the EU ICT-48 2020 project TAILOR (No. 952215), the PRIN project RIPER (No. 20203FFYLK), the JPMorgan AI Faculty Research Award "Resilience-based Generalized Planning and Strategic Reasoning”


\bibliographystyle{kr}
\bibliography{references}
 
\end{document}